\documentclass[letterpaper, 10 pt, conference]{ieeeconf}  

\IEEEoverridecommandlockouts                              

\usepackage{graphics} 
\usepackage{epsfig} 
\usepackage{amsthm}
\usepackage{amsmath} 
\usepackage{algorithm}
\usepackage{algpseudocode}
\usepackage{tabularx}
\usepackage{amsfonts}
\usepackage{makecell}
\usepackage{cite}
\makeatletter
\newcommand*{\rom}[1]{\expandafter\@slowromancap\romannumeral #1@}
\makeatother
\newtheorem{theorem}{Theorem}
\newtheorem{assumption}{Assumption}
\newtheorem{properties}{Property}
\newtheorem{definition}{Definition}

\DeclareMathOperator*{\argmax}{\arg\!\max}
\title{\LARGE \bf
Q-learning-based Model-free Safety Filter 
}

\author{Guo Ning Sue$^{*}$, Yogita Choudhary$^{*}$, Richard Desatnik, Carmel Majidi, John Dolan, Guanya Shi
\thanks{*Equal Contribution}
\thanks{All authors are from the Robotics Institute, Carnegie Mellon University \tt\footnotesize{\{gsue, ychoudha, rdesatni, cmajidi, jdolan, guanyas\}@andrew.cmu.edu}}}

\begin{document}

\maketitle
\thispagestyle{empty}
\pagestyle{empty}

\begin{abstract}
Ensuring safety via safety filters in real-world robotics presents significant challenges, particularly when the system dynamics is complex or unavailable. To handle this issue, learning-based safety filters recently gained popularity, which can be classified as model-based and model-free methods. Existing model-based approaches requires various assumptions on system model (e.g., control-affine), which limits their application in complex systems, and existing model-free approaches need substantial modifications to standard RL algorithms and lack versatility. This paper proposes a simple, plugin-and-play, and effective model-free safety filter learning framework. We introduce a novel reward formulation and use Q-learning to learn Q-value functions to safeguard arbitrary task specific nominal policies via filtering out their potentially unsafe actions.
The threshold used in the filtering process is supported by our theoretical analysis.
Due to its model-free nature and simplicity, our framework can be seamlessly integrated with various RL algorithms. We validate the proposed approach through simulations on double integrator and Dubin's car systems and demonstrate its effectiveness in real-world experiments with a soft robotic limb. 
\end{abstract}

\section{INTRODUCTION}

Reinforcement learning has shown tremendous progress in a variety of decision making and control problems, particularly games \cite{DBLP:journals/corr/abs-1903-00374}, locomotion \cite{margolis2022rapidlocomotionreinforcementlearning}, and robot manipulation \cite{DBLP:journals/corr/GuHLL16}. However, deploying RL in real-world robotic operations presents unique challenges, with safety being a paramount concern. Unfortunately, devising effective safety mechanisms is challenging. To tackle the issue of safety, various tools from control theory such as control barrier functions (CBFs) \cite{Ohnishi_2019, emam2022safereinforcementlearningusing, DBLP:journals/corr/abs-1903-08792}, model predictive safety certification \cite{wabersich2021probabilisticmodelpredictivesafety}, Hamilton-Jacobi reachability (HJR) \cite{fisac},
have been integrated with the standard reinforcement learning pipeline.
These methods require accurate models of robotic systems, which are often unavailable. To address this issue, recent research has shifted towards model-free approaches \cite{he2024agilesafelearningcollisionfree}. However, these approaches \cite{achiam2017constrainedpolicyoptimization}, \cite{tanyour} are either not robust, requiring near-perfect safety value function approximations \cite{castellano2024learningsafetycriticsnoncontractive}, or drastic changes to the standard model-free RL pipeline \cite{fisac2019bridging}. In this paper, we introduce a model-free safety filter that has theoretical safety guarantees in the optimal case, but is still robust to suboptimal conditions and fits seamlessly into the standard model-free RL pipeline.\\
\begin{figure}[!t]
    \centering
    \includegraphics[width=\columnwidth]{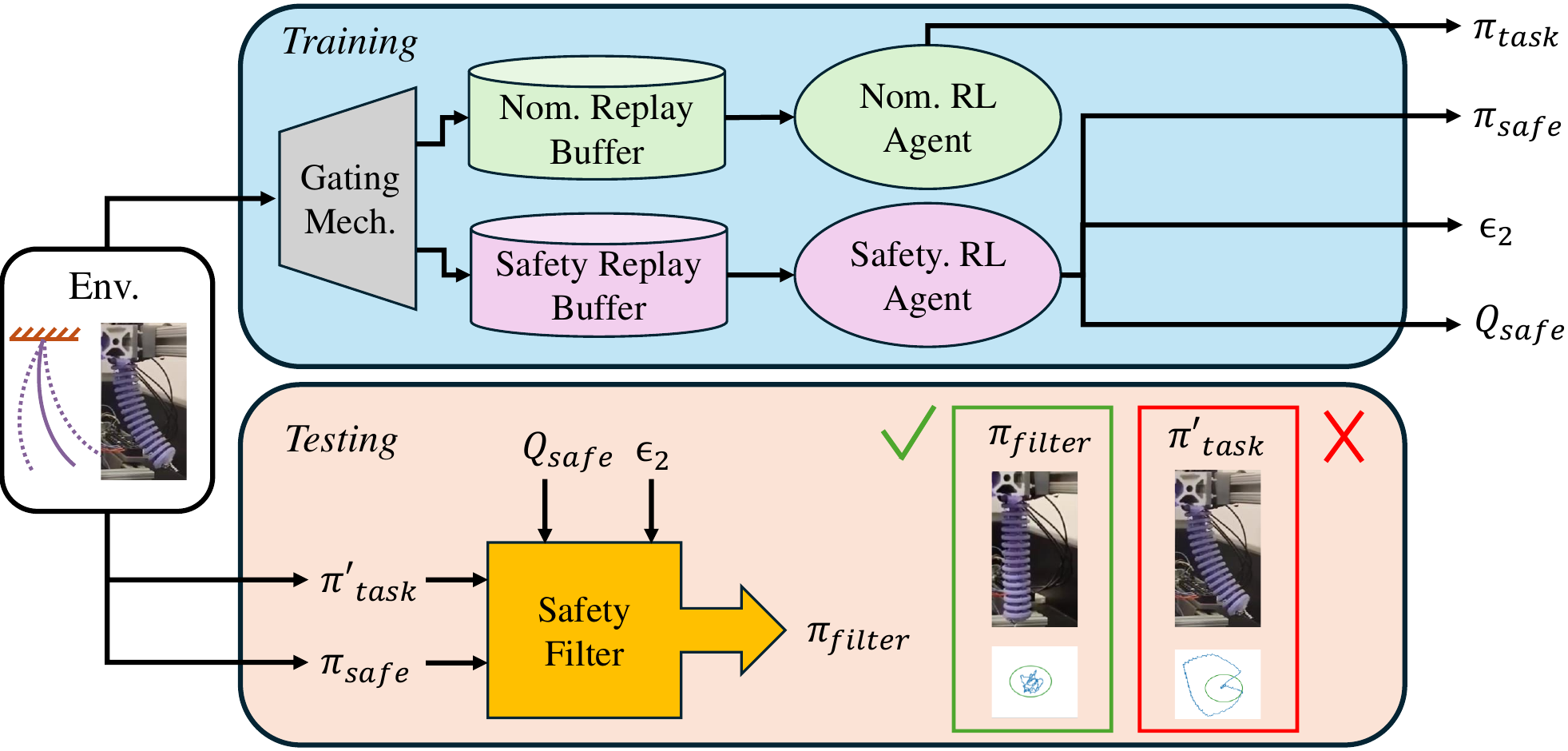} 
    \caption{The block diagram shows our model-free RL-based safety filter framework. During training, environment observations are stored in the replay buffers of both task specific nominal and safety agents, enabling their simultaneous training. In testing, observations are processed by both policies, and the nominal action is filtered based on the safety agent's Q-function and threshold $\epsilon_2$.
    }
    \label{fig:bloc}
\end{figure}

\subsection{Related Work}
\textbf{Model-free Safe Reinforcement Learning}: Model-free Safe Reinforcement Learning (RL) is frequently formulated as a Constrained Markov Decision Process (CMDP) \cite{altman-constrainedMDP}. A CMDP extends a standard MDP by incorporating an additional cost signal to identify state-action pairs that violate constraints. By defining cost thresholds, it is possible to derive policies with low failure probabilities. Popular methods for solving CMDPs include Lagrange multiplier methods  \cite{adaptivedp, tessler2018rcpo, Achiam2019BenchmarkingSE}, projection methods \cite{projectionbasedcpo}, and penalty function methods \cite{itegrateddecisionandcontrol}. 
Another line of work in model-free RL focuses on learning a safety critic to filter out unsafe actions and is also the inspiration behind our proposed approach. Safety Q-functions for Reinforcement Learning (SQRL) \cite{safetycritic_learningtobesafe} learns a safety critic that predict the future failure probability and uses the critic to constrain the nominal policy. The framework involves pre-training the safety critic and then fine-tuning the policy on target tasks using the learned safety precautions. 

Expanding on this approach, Recovery RL \cite{thananjeyan2021recovery} additionally learns a recovery policy along with safety critic. They utilize two separate policies for task and recovery to learn safely without compromising task performance. 

In contrast to \cite{thananjeyan2021recovery}, \cite{honari2024safetyoptimizedreinforcementlearning} focuses on jointly optimizing performance and safety. They modify the reward design proposed in \cite{thomas2022safereinforcementlearningimagining} by incorporating the safety critic into the reward to prevent exploration in unsafe regions during training. \cite{binarysafetycritic} introduces a binary safety critic by leveraging the idea that safety is a binary property. Our formulation also follows the safety critic structure but introduces a novel reward design to ensure safety.

\textbf{Model-based Safe Reinforcement Learning}: Model-based safe RL methods are more sample efficient compared to model-free methods. 
Recent studies integrate CMDPs with model-based RL to minimize training violations and speed up learning \cite{thomas2022safereinforcementlearningimagining, modelbased_cpoworldmodels, modelbased_cmbpo}. Safe Model-Based Policy Optimization (SMBPO) \cite{thomas2022safereinforcementlearningimagining} utilizes the learned model for planning to prevent safety violations by penalizing unsafe trajectories. Their approach, however assumes safety violations occur within a given horizon after entering irrecoverable states. Moreover, uncertainties in the learned model can lead to incorrect predictions, causing the agent to misclassify unsafe states as safe. In our proposed approach, we do not make any assumption on the horizon of irrecoverable states and utilize a model-free approach for ensuring safety.

\subsection{Contributions}
In this paper, we introduce a methodology to create a model-free safety filter that offers theoretical safety guarantees under optimal conditions, remains robust under suboptimal conditions, and seamlessly integrates with standard model-free reinforcement learning (RL) paradigms. An overview of our methodology is shown in Fig.~\ref{fig:bloc}. Our contributions are as follows: 
\begin{itemize}
    \item We propose a novel reward formulation that ensures safety throughout the entire episode length without assuming that all irrecoverable states lead to unsafe regions within a given time horizon, unlike previous approaches \cite{thomas2022safereinforcementlearningimagining, honari2024safetyoptimizedreinforcementlearning}.
    \item Our method integrates smoothly with existing RL paradigms, allowing simultaneous training of the task policy and safety filter using separate replay buffers. The safety filter generalizes to any task policy due to the decoupling of task and safety policies.
    \item We validate our approach on two classical control problems in simulation—the Dubin's car and double integrator systems—and perform real-world testing on a  soft robot. Our approach is robust against training inaccuracies and generalizes to systems with complex dynamics, such as soft robotic systems.
\end{itemize}


The paper is organized as follows. In Section~\ref{sec:prelim}, we introduce the definitions and terminologies commonly used in safe RL community. Section~\ref{sec:methodology} presents the novel reward formulation, detailed theoretical analysis, and implementation of the proposed approach. In Section~\ref{sec:Results}, we validate the approach through simulations on the double integrator system and Dubin's car, followed by hardware experiments on a soft robot. Finally, Section~\ref{sec:Conclusion} concludes the paper.

\section{Preliminaries}
\label{sec:prelim}

\subsection{State Space Partitioning}
Consider a car traveling along a road (Fig. \ref{fig:small_car}) where areas like the curb are designated as unsafe states (\( \mathcal{X}_{\text{unsafe}} \)) due to hazards. If the car is traveling at \( 80\,\text{km/h} \) and is only \( 1\,\text{m} \) from the curb,  a collision is inevitable even with maximum deceleration, defining this state as the irrecoverable state (\( \mathcal{X}_{\text{irrec}} \)). In contrast, if the car is \( 1\,\text{km} \) away from the curb, there are various control actions that can ensure safety, categorizing this state as an absolutely safe state (\( \mathcal{X}_{\text{safe}} \)). Therefore, we can divide the state space $\mathcal{X}$ into three subspaces: $\mathcal{X}_{\text{safe}}$, $\mathcal{X}_{\text{irrec}}$ and $\mathcal{X}_{\text{unsafe}}$.

\begin{definition}
$\mathcal{X}_{\text{safe}}$: The set of states for which there is always an action to keep out of the unsafe region and stay in the same set. For example a speeding car far from any obstacles.
\end{definition}
Mathematically, if $x_t \in \mathcal{X}_{safe}$ then $\exists u \in \mathcal{U}$ where $x_{t+1} = f(x_t, u)$ and $x_{t+1} \in \mathcal{X}_{safe}$. $f$ denotes system dynamics, $x_t$ denotes a state in the system at time $t$, $u$ stands for an action to take at $x$ and $\mathcal{U}$ denotes the set of possible actions to take. This implies that $\mathcal{X}_{\text{safe}}$ is \textit{forward invariant}.

\begin{definition}
$\mathcal{X}_{\text{irrec}}$: The set of states of which is currently not in $\mathcal{X}_{unsafe}$, but will inevitably reach the unsafe region despite all actions taken. For example, a speeding car very close to an obstacle. Due to its speed, despite applying a maximum breaking deceleration, there is nothing to prevent its collision.
\end{definition}
Mathematically, if $x_t \in \mathcal{X}_{irrec}$ then $x_{t + T} \in \mathcal{X}_{unsafe}$ for any $u_t, u_{t+1}....,u_{t+T} \in \mathcal{U}$ where $T$ is a positive integer.

\begin{definition}
$\mathcal{X}_{\text{unsafe}}$: The set of states that is human defined to be unsafe either to other humans or the system.
\end{definition}

Furthermore, these subspaces have the following relations:
\begin{align}
    \mathcal{X} &= \mathcal{X}_{\text{safe}} \cup \mathcal{X}_{\text{irrec}} \cup \mathcal{X}_{\text{unsafe}} \label{eq:comp_of_union} \\
    \mathcal{X}_{\text{unsafe}} \cap \mathcal{X}_{\text{safe}} &= \mathcal{X}_{\text{unsafe}} \cap \mathcal{X}_{\text{irrec}} = \mathcal{X}_{\text{safe}} \cap \mathcal{X}_{\text{irrec}} = \emptyset \label{eq:comp_of_X_intersection}
\end{align}

Identifying irrecoverable states is challenging due to system dynamics, but crucial as entering \( \mathcal{X}_{\text{irrec}} \) leads to \( \mathcal{X}_{\text{unsafe}} \). 
 \begin{figure}[!t]
    \centering
    \includegraphics[width=\columnwidth]{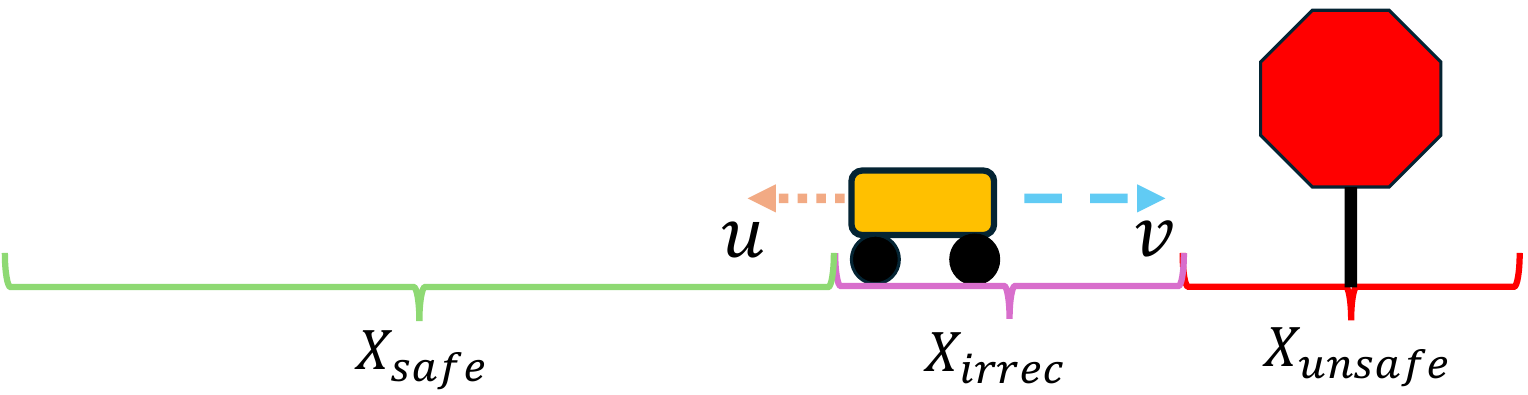} 
    \caption{The full state space is broken down into three regions, $\mathcal{X}_{safe}$, $\mathcal{X}_{irrec}$, $\mathcal{X}_{unsafe}$. $\mathcal{X}_{safe}$ is the region the control input $u$ can always be applied to prevent the system from entering $\mathcal{X}_{unsafe}$. $\mathcal{X}_{irrec}$ is the region where no control input can prevent entry into $\mathcal{X}_{unsafe}$. If a car is moving too fast and too close to the unsafe region, it is an example of the system being in the irrecoverable region.
    }
    \label{fig:small_car}
\end{figure}

\subsection{Q Learning}
We propose to use Q-learning for constructing a safety filter that filters hazardous actions and keeps the system in $\mathcal{X}_{safe}$. In Q-learning, the goal is to learn the expected discounted cumulative reward for a state-action pair in a Markov Decision Process (MDP) \cite{Sutton1998}. The Q-value is updated during training using the Bellman update rule, under deterministic dynamics:
\begin{equation}
    Q(x, u) = r(x, u, x') + \gamma \max_{u'} Q(x', u')
    \label{eq:bellman}
\end{equation}
where \(r\) is the reward function and \(\gamma\) is the discount factor. The value function is defined as:
\begin{equation}
    V(x) = \max_{u} Q(x, u)
\end{equation}
We propose learning optimal Q-function and value function \(Q^*_{safe}\) and \(V^*_{safe}\), respectively using a specialized reward function \(r_{safe}\), where states in \(\mathcal{X}_{safe}\) exceed a threshold \(\epsilon_2\). This threshold helps in determining whether actions from task specific nominal policies \(\pi_{task}\) are unsafe based on the learned Q and V functions, $\hat Q_{safe}$ and $\hat V_{safe}$ and replace with safer actions which provide the highest expected discounted cumulative safe reward.

\section{Methodology}
\label{sec:methodology}

In this work, we propose a Q-learning-based, model-free reinforcement learning approach to construct a safety filter. A block diagram of the framework is shown in Fig.~\ref{fig:bloc}. Our method leverages the fact that \(\hat Q_{safe}\) and \(\hat V_{safe}\) can be learned off-policy, enabling the simultaneous training of a task-specific policy \(\pi_{task}\) and the safety policy \(\pi_{safe}\). 

The two policies are trained in parallel but remain decoupled by a gating mechanism that sorts observations into separate replay buffers based on episodic conditions. This decoupling allows \(\pi_{task}\) to be swapped with any \(\pi'_{task}\) during testing. 
Our theoretical analysis focuses on the formulation of \(r_{safe}\) and its impact on \(\hat Q_{safe}\), \(\hat V_{safe}\), and \(\pi_{safe}\), as \(\pi_{task}\) can be any task policy.

\subsection{Formulation and Analysis}
\subsubsection{Reward Formulation}
We define the safety reward function \( r_{\text{safe}}(x, u, x') \) as:
\begin{equation}
\label{eq:rewards}
    r_{\text{safe}}(x, u, x') = \begin{cases}
        l(x), & x, x' \notin \mathcal{X}_{\text{unsafe}} \\
        -\dfrac{1}{\gamma^t(1 - \gamma)}, & x \notin \mathcal{X}_{\text{unsafe}},\; x' \in \mathcal{X}_{\text{unsafe}} \\
        -1, & x, x' \in \mathcal{X}_{\text{unsafe}}
    \end{cases}
\end{equation}
where \( x \) is the state comprising of time \( t \), position and velocity. \( x' \) is the state the system transitions to after applying control \( u \), and \( \gamma \in (0, 1) \) is the discount factor. The function \( l(x) \in (0, 1] \) increases as the system moves deeper into the safe region; for example:

\begin{equation}
    l(x) = \frac{d(x)}{\max_{x \notin \mathcal{X}_{\text{unsafe}}} d(x)},
\end{equation}
where \( d(x) \) is the positive signed distance to \( \mathcal{X}_{\text{unsafe}} \). For $x$ where $d(x) = 0$, i.e. on the boundary of $\mathcal{X}_{unsafe}$, we consider $x \in \mathcal{X}_{unsafe}$

During training, we optimize the $\hat Q_{\text{safe}}$ function using standard Bellman updates using \( r_{\text{safe}} \).

\begin{assumption}
\label{assump:early_term}
    An episode terminates at the time step when the agent reaches the unsafe region or at maximum episode length, $T$.     
\end{assumption}
We consider a finite episode length scenario as from Eq.~\ref{eq:rewards}, when \( x \in \mathcal{X}_{\text{safe}} \) and \( x' \in \mathcal{X}_{\text{unsafe}} \), the reward can become unbounded if \( t \) is very large. Therefore, we define \( T \) as the episode length during training to bound the reward and prevent divergence. 

Our design of \( r_{\text{safe}} \) aims to ensure a clear separation on $\hat V_{safe}$ between states in \( \mathcal{X}_{\text{safe}} \), and those in \( \mathcal{X}_{\text{irrec}} \). While an intuitive approach is to add a large negative penalty upon entering the unsafe region \cite{thomas2022safereinforcementlearningimagining}, fixed penalties diminish over time due to discounting, blurring the value separation between \( \mathcal{X}_{\text{safe}} \) and \( \mathcal{X}_{\text{irrec}} \). 
This reward formulation ensures a clear separation between \( \mathcal{X}_{\text{safe}} \) and \( \mathcal{X}_{\text{irrec}} \) in discounted cumulative rewards shown in the properties described below.

\subsubsection{Analysis}
\label{sec:anaylsis}
The true value function $V^*_{safe}$ has the following properties.

\begin{properties}
\label{prop:safe_states}
    If $x \in \mathcal{X}_{safe}$, then the optimal value of the value function satisfies the following: $0 < V^*_{safe}(x) \leq \frac{1-\gamma^{T+1}}{1 - \gamma}$.
\end{properties}

From the definition of $\mathcal{X}_{\text{safe}}$, $\mathcal{X}_{\text{safe}}$ is forward invariant and thus the agent accumulates the discounted reward $l(x)$ for $T$ steps. Since, $ 0 < l(x) \le 1$, therefore:

\begin{equation*}
    \sum^{T}_{i = 0} \gamma^i\cdot 0 <  V^*_{safe}(x) = \sum^{T}_{i = 0} \gamma^il(x) \le \sum^{T}_{i = 0} \gamma^i\cdot1
\end{equation*}
\begin{equation}
    0 < V^*_{safe}(x) \le \frac{1-\gamma^{T+1}}{1 -\gamma}
\end{equation}

Since $l(x)$ is continuous and increases as the system moves away from $\mathcal{X}_{unsafe}$, so are values of $V^*_{safe}$.

\begin{properties}
\label{prop:irrec_states}
    If $x \in \mathcal{X}_{irrec}$, then the optimal value of the value function satisfy the following $-\frac{1+\gamma}{1 - \gamma} \leq V_{safe}^*(x) < 0$.     
\end{properties}

In the irrecoverable region, episodes terminate in a finite number of steps with a large negative reward upon entering the unsafe region. Therefore, for \( x \in \mathcal{X}_{\text{irrec}} \), the optimal value function \( V^*_{\text{safe}}(x) \) is represented by the following summation:

\begin{equation*}
    V_{irrec}(x) = -\sum^{T}_{i = N+1}{\gamma^i} - \frac{\gamma^N}{\gamma^N(1-\gamma)} + \sum^{N-1}_{i = 0} \gamma^il(x)
\end{equation*}

where $N$ is a positive integer representing the last time step in $\mathcal{X}_{safe}$. Since $ 0 < l(x) \le 1$, substituting and simplifying:
\begin{equation*}
\label{eq:ineq}
    \frac{-1 + \gamma^{T+1} - \gamma}{1 - \gamma} \le V^*_{safe}(x) \le \frac{-{\gamma}^{T}}{1 - \gamma}
\end{equation*}
\begin{equation}
\label{eq:irrec}
    \frac{-1 + \gamma^{T+1} - \gamma}{1 - \gamma} \le V^*_{safe}(x) < 0
\end{equation}
as $\gamma \in (0, 1)$

\begin{properties}
    \label{prop:unsafe_states}
    If $x \in \mathcal{X}_{unsafe}$, then the optimal value of the value function satisfy the following $V^*_{safe}(x) = -\frac{1-\gamma^{T+1}}{1-\gamma}$.
\end{properties}
When $x \in \mathcal{X}_{unsafe}$, the episode terminates immediately with a penalty of $-1$. Therefore, this terminal state accumulates the discounted cumulative reward of $-1$ for $T$ steps:
\begin{equation}
    V^*_{safe}(x) = -\sum^{T}_{i = 0} \gamma^i
    = -\frac{1-\gamma^{T+1}}{1-\gamma}
\end{equation}

These properties show that \( V^*_{\text{safe}}(x) \) generally increases from \( \mathcal{X}_{\text{unsafe}} \) to \( \mathcal{X}_{\text{safe}} \),  especially the transition from \( \mathcal{X}_{\text{irrec}} \) to \( \mathcal{X}_{\text{safe}} \). This indicates that the boundary between \( \mathcal{X}_{\text{irrec}} \) and \( \mathcal{X}_{\text{safe}} \) corresponds to a threshold value distinguishing these two regions in \( V^*_{\text{safe}}(x) \). Fig. \ref{fig:V_func} depicts a conceptual visualization of the \( V^*_{\text{safe}}(x) \). 

\begin{figure}[!t]
    \centering
    \includegraphics[width=\columnwidth]{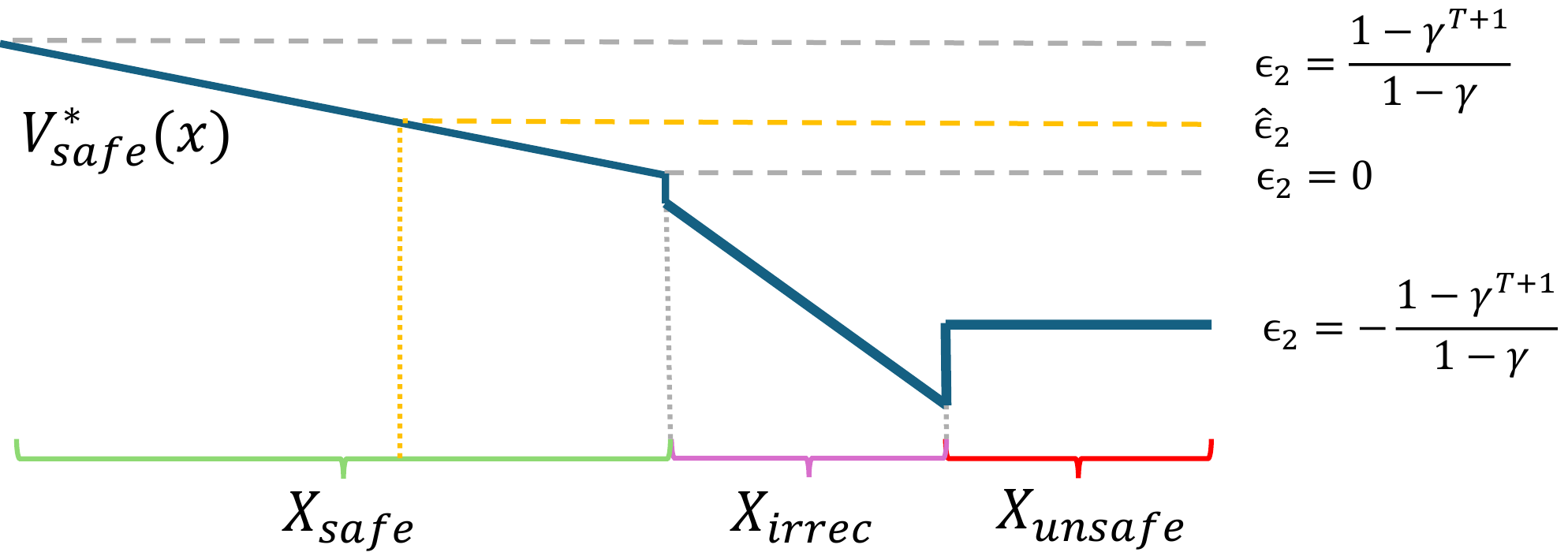} 
    \caption{A 1 Dimensional concept visualization of how the $V^*_{safe}$ could be like. $\epsilon_2=0$ is the threshold value that separates $\mathcal{X}_{\text{safe}}$ and $\mathcal{X}_{\text{irrec}}$. Because $V^*_{safe}$ is increasing deeper inside $\mathcal{X}_{\text{safe}}$, picking a threshold value $\hat \epsilon_2$ that is higher than the optimal $\epsilon$, will gives a more conservative estimate of the safe set.
    }
    \label{fig:V_func}
\end{figure}

From Eq.~\ref{eq:comp_of_X_intersection} and from Prop.~\ref{prop:safe_states}:
\begin{equation}
\label{eq:safe}
    x \in \mathcal{X}_{safe} \iff V_{safe}^*(x) \in [0, \frac{1-\gamma^{T+1}}{1-\gamma}]
\end{equation}
Again from Eq. \ref{eq:comp_of_X_intersection}, if $x \notin \mathcal{X}_{safe}$, it must be in $\mathcal{X}_{irrec}$ or $\mathcal{X}_{unsafe}$:
\begin{equation}
\label{eq:unsafe}
    V^*_{safe}(x) < 0 \implies x \in \mathcal{X}_{irrec} \text{ or } x \in \mathcal{X}_{unsafe}
\end{equation}

Integrating the insights from Eq.~\ref{eq:safe}, Eq.~\ref{eq:unsafe}, and the fact that $V^*_{safe}(x) = \max_{\mathbf{u} \in \mathcal{U}} Q^*_{safe}(x, u)$, where $\mathcal{U}$ denotes the set of permissible control inputs, leads to the following action filtering scheme to ensure safety:
\begin{equation}
\label{eq:filter}
    \pi_{filter} (x) =
    \begin{cases}
      \pi_{task}(x),  & Q_s^*(x, \pi_{task}(x)) > \epsilon_2\\
      \argmax\limits_{u \in \mathcal{U}} Q_s^*(x, u),  & Q_s^*(x, \pi_{task}(x)) \le \epsilon_2
    \end{cases}     
\end{equation}
where $\epsilon_2$ denotes the value threshold that separates $\mathcal{X}_{safe}$ and $\mathcal{X}_{unsafe}$. $Q_s^*$ represents $Q_{safe}^*$, $\pi_{task}$ is any policy trained to achieve a specific task under arbitrary task rewards $r_{task}$, and $\argmax_{\mathbf{u} \in \mathcal{U}} Q_s^*(x, \mathbf{u})$ is the safety policy $\pi_{safe}$. $\pi_{filter}$ refers to the filtered policy that ensures task completion while maintaining safety. For the optimal $V^*_{safe}$, we have $\epsilon_2 = 0$.

\begin{theorem}
\label{thm:forward_invariance}
    Given $Q^*_{safe}$ and if $x \in \mathcal{X}_{safe}$ and $\pi_{filter}(x)$ is followed, then $x' \in \mathcal{X}_{safe}$ where $x'$ is the state of the system after taking $\pi_{filter}(x)$ from $x$.
\end{theorem}

\begin{proof}
    We proceed to prove by contradiction, assuming first that $x' \notin \mathcal{X}_{safe}$. Letting $\pi_{filter}(x) = \mathbf{a}$, from the Bellman Optimality equation of Q-Learning:
    \begin{align*}
        Q^*_{safe}(x, \mathbf{a}) &= r_{safe}(x, a, x') + \gamma \max_{u \in \mathcal{U}}Q^*_{safe}(x', u) \\
        &= r_{safe}(x, a, x') + \gamma V^*_{safe}(x')
    \end{align*}
    Since $x' \notin \mathcal{X}_{safe}$, that implies:
    \begin{equation*}
        r_{safe}(x, a, x') < 0 ;\;  \gamma V^*_{safe}(x') < 0 ;\;        Q^*_{safe}(x, a) < 0
    \end{equation*}
    From conditions outlined for $\pi_{filter}(x)$ in Eq.\ref{eq:filter}:
    \begin{equation*}
        a = \argmax\limits_{u \in \mathcal{U}} Q^*_{safe}(x, u)
    \end{equation*}
    However
    \begin{align*}
        Q^*_{safe}(x, a) &= Q^*_{safe}(x, \argmax\limits_{u \in \mathcal{U}} Q^*_{safe}(x, u)) \\
        &= \max_{\mathbf{u} \in \mathcal{U}}Q^*_{safe}(x, \mathbf{u}) = V^*_{safe}(x) < 0
    \end{align*}
    According to Eq. \ref{eq:unsafe}, $x \in \mathcal{X}_{irrec}$ or $x \in \mathcal{X}_{unsafe}$, which is a contradiction.
\end{proof}

Theorem \ref{thm:forward_invariance} demonstrates that the filtering scheme described in Eq.~\ref{eq:filter} selects actions that keep the system within $\mathcal{X}_{safe}$. Moreover, the scheme intervenes only when the system is about to exit $\mathcal{X}_{safe}$, thereby minimally impacting the performance of the task policy.

\subsection{Implementation}
To implement \( \pi_{\text{filter}} \) and keep the system within \( \mathcal{X}_{\text{safe}} \), the task policy action is replaced by a greedy policy that maximizes the learned \( \hat{Q}_{\text{safe}} \) under \( r_{\text{safe}} \). We use Soft Actor-Critic (SAC) \cite{haarnoja2018softactorcriticoffpolicymaximum} for simulations with continuous action spaces and Deep Q-Learning (DQN) \cite{MnihKSGAWR13} for real-world validation.

\subsubsection{Simultaneous Training of Task Policy and Safety Policy}

Leveraging the off-policy nature of SAC and DQN agents, we simultaneously train a task agent that maximizes task reward and a safety agent that maximizes the proposed safety reward. Both agents share the environment and rolled-out data but maintain independent replay buffers. A gating mechanism prevents the safety agent from collecting observations once the system enters the unsafe region. Initially, observations are added to both agents' replay buffers with their respective rewards. When the safety agent reaches the unsafe region, it stops receiving new observations, while the nominal agent continues until the episode ends. This approach minimizes interference with the task agent, preserving task performance, and ensures the safety policy remains valid for any task policy despite simultaneous training and shared data.

Due to Assumption \ref{assump:early_term}, the safety agent does not observe the unsafe region, as the episode terminates once the unsafe region is reached. Therefore a supervised loss, similar to the loss used in \cite{tanyour}, is introduced to the safety agent  in addition to the standard RL losses to classify whether a state belongs to $\mathcal{X}_{unsafe}$. The loss is defined as follows:
\begin{equation}
    \mathcal{L}_{unsafe} = \| \hat V_{safe}(x)  + \frac{1-\gamma^{T+1}}{1-\gamma} \|, \; \forall x \in \mathcal{X}_{unsafe}
\end{equation}
where $\hat V_{safe}$ denotes the learned value function.

\subsubsection{Tuning the threshold}

Converging to \( V^*_{\text{safe}} \) is challenging, and while \( \epsilon_2 = 0 \) is optimal in theory, it may not represent the boundary for \( \hat{V}_{\text{safe}} \). Our filter also loses guarantees in suboptimal cases. However, if \( \hat{V}_{\text{safe}} \approx V^*_{\text{safe}} \), a threshold separating \( \mathcal{X}_{\text{safe}} \) and \( \mathcal{X}_{\text{irrec}} \) should exist due to the increasing nature of \( V^*_{\text{safe}} \).

As Fig. \ref{fig:V_func} shows, raising \( \epsilon_2 \) gives more conservative safe region estimates, leading \( \pi_{\text{filter}} \) to select safer actions, reducing the risk of irrecoverable states. Starting from $\epsilon_2=0$, one can tune how conservative the learned filter will act, enabling practical use of suboptimal \( \hat{Q}_{\text{safe}} \) in the safety filter. Section 4 demonstrates that increasing the threshold ensures safety despite discrepancies between the learned and true value and Q functions.

\section{Results}
\label{sec:Results}
We experimentally validate our method in simulations and on real hardware. Simulations involve a double integrator system for basic validation and a Dubin's car environment for static obstacle avoidance and to compare with existing methods. On real hardware we demonstrated the filter's capability to constrain a soft robotic limb within a specified region.

\subsection{Simulation Environments}

\subsubsection{Double Integrator}
The double integrator state is two-dimensional, with position \( x \) and velocity \( \dot{x} \). Safe bounds are \( \pm 2~\text{m} \) and \( \pm 3~\text{m/s} \), with absolute maximum bounds of \( \pm 4~\text{m} \) and \( \pm 5~\text{m/s} \). The input acceleration is bounded between \( \pm 2~\text{m/s}^2 \), and the task is to reach a goal at \( 1.8~\text{m} \) (neon green dashed line in Fig.~\ref{fig:V_func_heatmap}).

The two-dimensional state allows easy visualization of the learned value function to empirically validate our theory. In Fig.~\ref{fig:V_func_heatmap}, safety policy actions (acceleration) are shown on the left, and the value function on the right. The solid black contour represents the analytical safe set  \cite{fisac2019bridging}. The black dashed lines denote the learned safe set where \( \epsilon_2 = 0 \), which lies entirely within the analytical safe set. The blue dashed line represents the contour for \( \epsilon_2 = 90 \); as per our formulation, this contour is smaller than the \( \epsilon_2 = 0 \) contour, leading to more conservative estimates of \( \mathcal{X}_{\text{safe}} \), which is observed empirically.

Moreover, on the left of Fig.~\ref{fig:V_func_heatmap}, at the boundary of the analytical safe set, the actions align with intuition: when approaching the right boundary at high speed, the action is a leftward force (red), and vice versa. These empirical results are consistent with our theoretical expectations.

\begin{figure}[!t]
    \centering
    \includegraphics[width=\columnwidth]{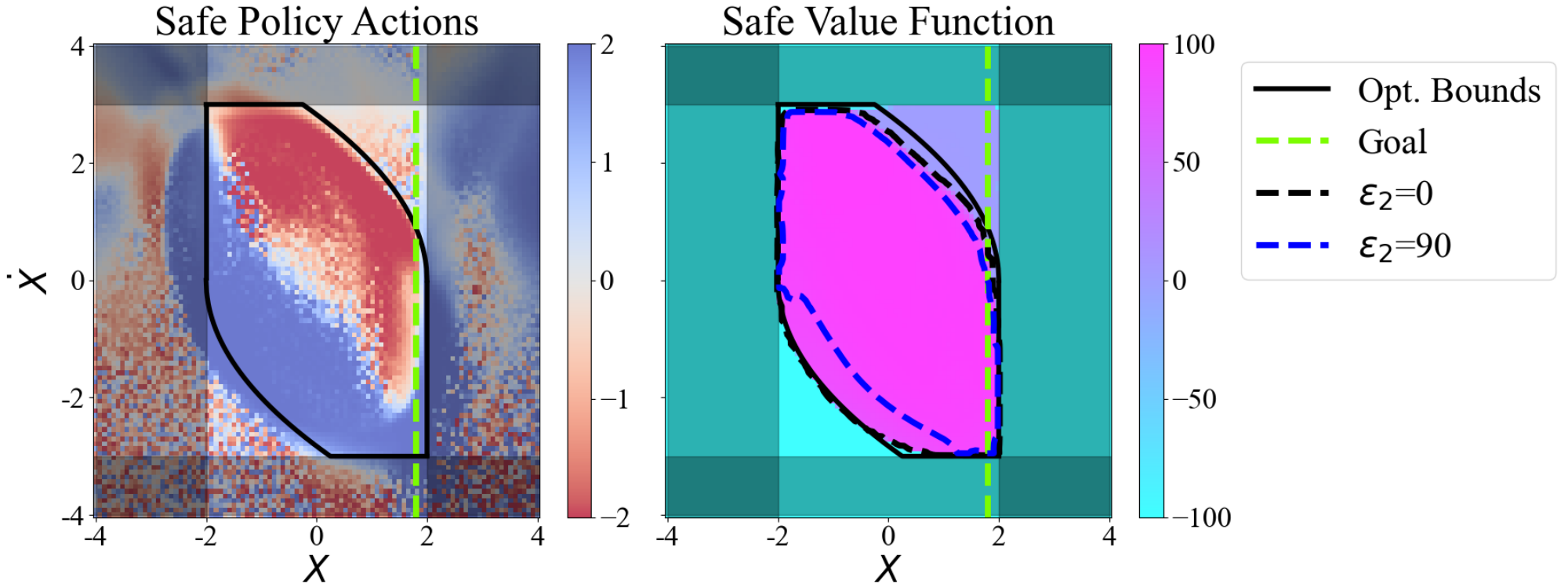} 
    \caption{\textit{Left}: The actions of the safe policy as a function of the state space. \textit{Right}: The value function of the safety agent for double integrator. The dark shaded area is the unsafe region.
    }
    \label{fig:V_func_heatmap}
\end{figure}

\subsubsection{Dubin's car}
Because the double integrator's goal lies within the safe region, the nominal agent remains safe by simply achieving the task, which does not fully demonstrate our safety filter's impact. Therefore, we tested our method in the Dubin's car environment, featuring complex nonlinear dynamics and conflicts between the nominal task and safety, e.g. the shortest path to goal goes through $\mathcal{X}_{unsafe}$.  In this environment, safety bounds are \( \pm 2~\text{m} \) with a central keep-out region of radius \( 1~\text{m} \). The absolute maximum bounds are defined with an input velocity of \( 1.2~\text{m/s} \), aiming for a goal at \( (1.8~\text{m},\, 1.8~\text{m}) \) with a radius of \( 0.5~\text{m} \). As shown in Tab.~\ref{tab:basline_comparisons}, our method's performance is comparable to, and sometimes better than, other common safe RL algorithms.

In Tab.~\ref{tab:basline_comparisons}, we compare our filtered policy (safety filter with co-trained task policy) against Lagrangian Relaxation (LR), Risk Sensitive Policy Optimization (RSPO) \cite{thananjeyan2021recovery}, the unconstrained task policy, and the task policy with a large penalty for entering the unsafe region (Reward Penalty). Results are the mean and standard deviation over 10 runs with different seeds. Using \( \epsilon_2 = 67.38 \), the result of a generic hyper parameter sweep, for our safety threshold, our method attains the highest average episodic return among methods achieving a 100\% safety rate (fraction of episodes reaching the maximum length without safety violations out of 100 episodes),.

We also tested our filter with task policies not co-trained with our safety agent (Tab.~\ref{tab:safety_diff_nom}): PPO \cite{ppo}, DDPG \cite{lillicrap2019continuouscontroldeepreinforcement}, TD3 \cite{fujimoto2018addressingfunctionapproximationerror}, and a uniform random policy. Over 10 runs with different seeds, only PPO had safety violations, with an average safety rate of 98.8\%. We attribute this minor unsafety to the stochastic nature of and our safety filter trained with SAC as it samples from a learned distribution to output an action.

As shown in the left of Fig.~\ref{fig:eval_baselines_eps}, there is a clear trend between \( \epsilon_2 \) and performance. Increasing \( \epsilon_2 \) results in greater safety but reduced nominal rewards, aligning with our theoretical prediction that higher \( \epsilon_2 \) leads to a more conservative method. The blue cross indicates \( \hat{\epsilon}_2 \), the value best value from a hyper parameter sweep. 

\begin{figure}[!t]
    \centering
\includegraphics[width=\columnwidth]{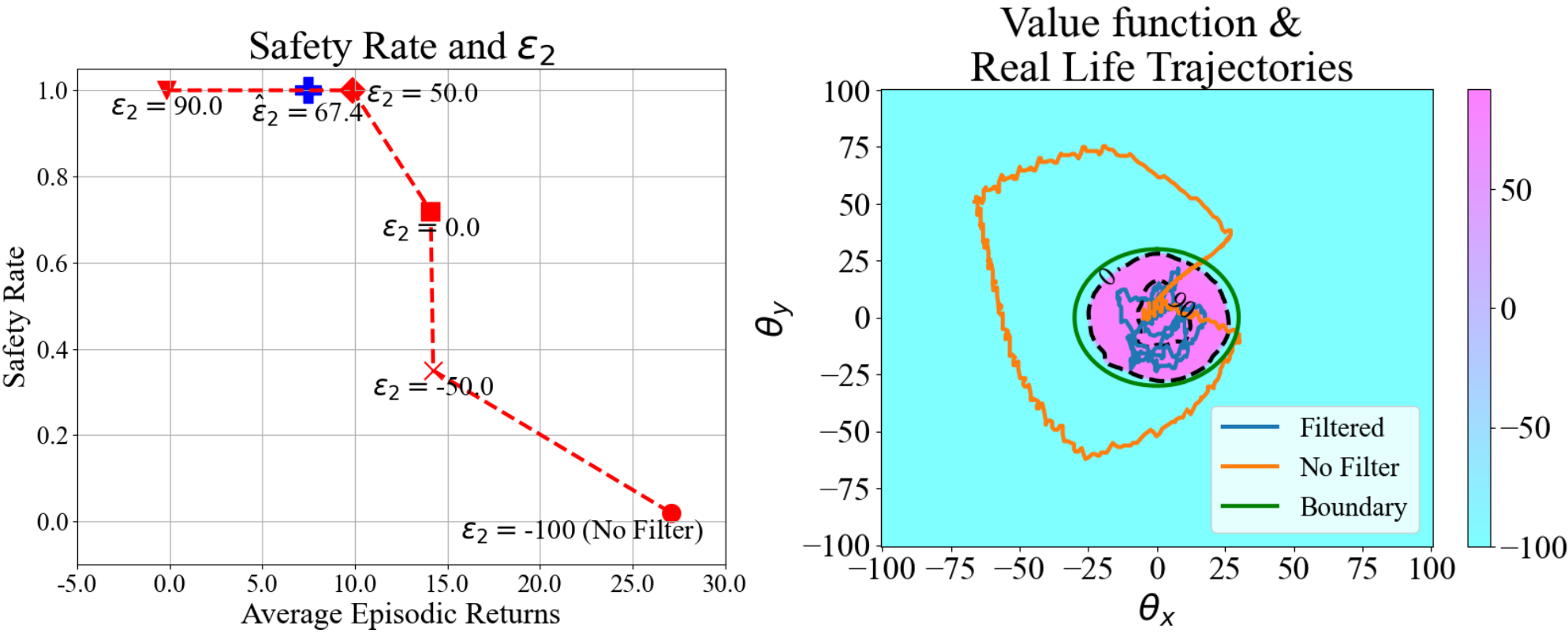} 
    \caption{\textit{Left}: Performance of our filter evaluated by average episodic return and safety rate at different $\epsilon_2$ levels in the Dubin's car environment. \textit{Right}: Real-life limb trajectory overlaid on the learned value function. The value function is a 2D projection (with velocities set to zero) of a 4D function learned through simulation, while trajectory values are recorded in real life. Dark dashed lines denote the 0 and 90 threshold contours.}
    \label{fig:eval_baselines_eps}
\end{figure}

\begin{table}[h!]
\centering
\caption{Performance comparison of our safety filter to other safe model-free reinforcement learning methods.}
\begin{tabular}{|c|c|c|}
\hline
\textbf{} & \textbf{Average Episodic Return} & \textbf{Safety Rate} \\ \hline
\makecell{\textbf{Unconstrained} \\ \textbf{Task Policy}} & 46.38 ± 20.37 & 0.017 ± 0.051 \\ \hline
\textbf{RSPO} & 13.37 ± 4.61 & 0.855 ± 0.138 \\ \hline
\textbf{Reward Penalty} & -2.10 ± 8.49 & 1.000 ± 0.000 \\ \hline
\textbf{LR} & 2.84 ± 2.34 & 1.000 ± 0.000 \\ \hline
\textbf{Filtered Policy} & 7.08 ± 3.37 & 1.000 ± 0.000 \\ \hline
\end{tabular}
\label{tab:basline_comparisons}
\end{table}

\begin{table}[h!]
\centering
\caption{Safety Rate of filtering different policies}
\begin{tabular}{|c|c|c|}
\hline
\textbf{} & \textbf{Safety Rate} & \textbf{Std. Safety Rate} \\ \hline
\textbf{Co-trained Policy} & 1.0 & 0.0 \\ \hline
\textbf{PPO} & 0.988 & 0.023 \\ \hline
\textbf{DDPG} & 1.0 & 0.0 \\ \hline
\textbf{TD3} & 1.0 & 0.0 \\ \hline
\textbf{Random} & 1.0 & 0.0 \\ \hline
\end{tabular}
\label{tab:safety_diff_nom}
\end{table}

\subsection{Hardware Experiments}
\subsubsection{Hardware Setup} 
To experimentally validate the proposed scheme, we implemented it on a soft robotic limb similar to the setup used in \cite{softlimb} and is depicted in Fig.~\ref{fig:hardware_setup}. The soft limb is injection-molded from silicone (Smooth-Sil 945) and actuated by four shape memory alloy (SMA) coils (Flextinol) positioned at the up, down, left, and right orientations. The coils contract when heated by a Pulse Width Modulated (PWM) controlled electrical current that are applied at the coil tips. When the temperature rises above $90^\circ\mathrm{C}$, due to the phase transition of the SMA, the limb bend in a specific direction, increasing the bending angle.

A capacitive bend sensor (Bend Labs Digital Flex Sensor - 2-Axis, 4 Inch) is located in the center to detect the bending angles of the limb. The bending angles $(\theta_x, \theta_y)$ are defined as the angle between the tangent at the limb tip and the vertical, measured along both the $x$ and $y$ axes. 
Since the soft limb's actuation depends on heat—increased heat leads to greater contraction—and due to the malleable nature of silicone, the system exhibits highly nonlinear dynamics. Furthermore, we control only the PWM signals that induce heating to raise the temperature, but relying only on passive cooling to lower it. These factors complicate the dynamics, rendering an analytical mapping from PWM signals to limb states intractable.
\begin{figure}[!t]
    \centering
\includegraphics[width=\columnwidth]{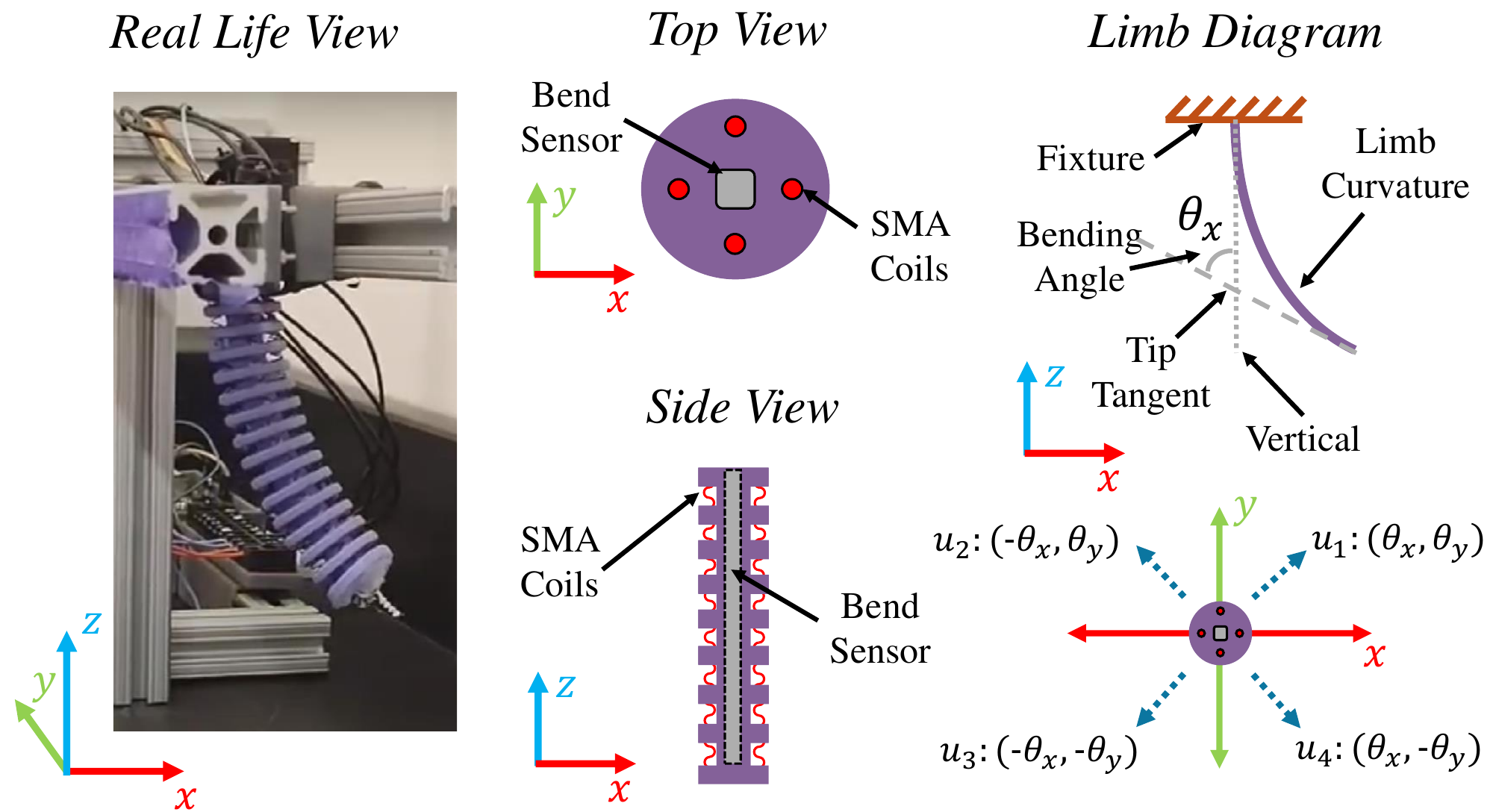} 
    \caption{Experimental setup of the soft robotic limb. \textit{Left}: Soft limb mounted on an aluminum fixture, actuated by SMA coils controlled via PWM through an Arduino UNO. \textit{Middle}: Cross-sectional view (not to scale) showing the embedded two-axis bending angle sensor (gray) in the silicone limb (purple), with SMA coils at cardinal directions. \textit{Right}: Diagram illustrating the bending angle and the four actions used by the trained policy. Video of hardware results available at: https://tinyurl.com/mtfxnc9r}
    \label{fig:hardware_setup}
\end{figure}

\subsubsection{Data Setup}
To enable efficient training, instead of naively deploying RL training on the hardware, we first create a data driven simulator to approximate the system dynamics. Data was collected through random motor wobbling. However, due to the complexity of the dynamics, the best dynamic model we were able to train had a root mean square error (RMSE) of approximately $15^\circ$.

\subsubsection{RL Setup}
To validate our method, we trained the safety policy using DQN with the observation space \( x = [\theta_x, \theta_y, \dot{\theta}_x, \dot{\theta}_y] \) and an action space of four actions—move upper right (\( u_1 \)), upper left (\( u_2 \)), lower left (\( u_3 \)), and lower right (\( u_4 \))—as depicted in Fig.~\ref{fig:hardware_setup}. The safety policy was trained using our approximated dynamics simulator to constrain movements within \( \pm 30^\circ \) on both \( x \) and \( y \) axes. The constraint boundary is shown as a green circle in the right figure of Fig.~\ref{fig:eval_baselines_eps}, i.e., \( \mathcal{X}_{\text{safe}} \) represents states where \( \|\theta\| < 30^\circ \).

For the task policy, we used a predefined action sequence: \( u_1 \) for 3 seconds, \( u_4 \) for 3 seconds, \( u_3 \) for 6 seconds, \( u_2 \) for 12 seconds, \( u_1 \) for 6 seconds, and \( u_3 \) for 3 seconds, aiming to trace a diamond-shaped trajectory. The real-life rollout of this sequence is shown in orange on the right of Fig.~\ref{fig:eval_baselines_eps}. The task of the safety filter is to constraint the limb movement within the \( \|\theta\| < 30^\circ \) boundary.

\subsubsection{Hardware Results}
As shown in the right of Fig.~\ref{fig:eval_baselines_eps}, the safety policy successfully keeps the limb within the safe region, as the blue trajectories, which represent the recorded limb trajectory with our safety filter, are entirely contained within the safe region denoted by the green circle. However, due to the sim-to-real gap between the learned dynamics and the real limb dynamics, $\epsilon_2 = 90$ was required for our scheme to work in real life, resulting in a very conservative filter. Despite this conservatism, the empirical results validate the robustness of our method in maintaining system safety, even in the presence of a significant sim-to-real gap.

\section{Conclusion} 
\label{sec:Conclusion}
We introduce a method to identify safe, unsafe, and irrecoverable states through reward shaping in standard reinforcement learning. We devise a policy filter with theoretical safety guarantees under optimal conditions, which, with threshold tuning, also performs well under suboptimal conditions like large sim-to-real gaps.

Our method allows simultaneous training of a task-specific policy and a safety policy that generalizes to any task policy. It matches existing methods in simulations and is validated on a real-world, highly nonlinear soft silicone limb, making it more versatile, easier to train, and more general than previous methods.

Unfortunately, while the independence of task and safety policy in our framework aids generalization, it may hinder the task performance when safety and task objectives conflict greatly, potentially causing deadlocks. Future research could focus on developing a filter that retains generalizability without significantly impeding the task policy. Additionally, although our method is robust to suboptimalities, it lacks safety guarantees under suboptimal conditions. Extending the formulation to provide safety guarantees even in suboptimal scenarios is another valuable research direction.

\addtolength{\textheight}{-12cm}   


\bibliographystyle{IEEEtran}
\bibliography{root}

\end{document}